\documentclass[letterpaper, 10 pt, conference]{ieeeconf}  

\usepackage{graphicx} 

\usepackage{amsthm}
\usepackage{amsmath}
\usepackage{subfigure}
\usepackage{cite}
\usepackage{enumerate}                                               

\IEEEoverridecommandlockouts                              
\title{\LARGE \bf
Swarm Bug Algorithms for Path Generation in Unknown Environments
}

\author{Alexander Johansson and Johan Markdahl
	\thanks{Alexander Johansson and Johan Markdahl are with the Swedish Defence Research Agency (FOI),  SE-100 44 Stockholm, Sweden. Emails: \tt\small alexander.samimi.johansson@foi.se, markdahl@kth.se  }
}

\newtheorem{theorem}{Theorem}         
         
\newtheorem{corollary}{Corollary} 
\newtheorem{algorithm}{Algorithm}
\newtheorem{remark}{Remark}
\newtheorem{proposition}{Proposition}
\begin{document}

\maketitle
\thispagestyle{empty}
\pagestyle{empty}

\begin{abstract}
	In this paper, we consider the problem of a swarm traveling between two points as fast as possible in an unknown environment cluttered with obstacles. Potential applications include search-and-rescue operations where damaged environments are typical. We present swarm generalizations, called SwarmCom, SwarmBug1, and SwarmBug2, of the classical path generation algorithms Com, Bug1, and Bug2. These algorithms were developed for unknown environments and require low computational power and memory storage, thereby freeing up resources for other tasks. We show the upper bound of the worst-case travel time for the first agent in the swarm to reach the target point for SwarmBug1. For SwarmBug2, we show that the algorithm underperforms in terms of worst-case travel time compared to SwarmBug1. For SwarmCom, we show that there exists a trivial scene for which the algorithm will not halt, and it thus has no performance guarantees. Moreover, by comparing the upper bound of the travel time for SwarmBug1 with a universal lower bound for any path generation algorithm, it is shown that in the limit when the number of agents in the swarm approaches infinity, no other algorithm has strictly better worst-case performance than SwarmBug1 and the universal lower bound is tight.

\end{abstract}

\section{INTRODUCTION}

\subsection{Motivation}

Swarm robotics has received increasing attention over the past two decades and is predicted to be disruptive in many fields \cite{Dorigo2020}. In many swarm robotics applications, the size of the robots is of critical importance, such as for search missions in narrow and unknown environments. This enforces hardware constraints on the robots, limiting the robots' computational power and memory storage. The most popular navigation solutions for unknown environments require computational power and memory storage that exceeds the hardware limits for small robots, for example, simultaneous localization and mapping (SLAM) \cite{Durrant2006}. Even if it is possible to run SLAM on small robots such as the Crazyflie platform \cite{Vikgren2023}, it consumes a large portion of the available resources that are better left available for other tasks. One  navigation solution class, bug algorithms, was developed for single-robot systems in the '80s when the hardware capacity was minimal compared to today. It is thus promising to explore bug algorithms as navigation solutions for swarm robotics.

\subsection{Related work}

The authors in the seminal work \cite{Lumelsky1986} and \cite{Lumelsky1987} presented the first bug algorithms. The purpose of bug algorithms is to generate paths between a start point and a target point by simple ``bug-like" behaviors such as ``wall following" and moving straight toward the target. Three algorithms, called Com, Bug1, and Bug2, with different levels of greediness and memory storage, were presented in \cite{Lumelsky1986} and \cite{Lumelsky1987}.  The algorithm Com (named after its ``common sense" behavior) is based on traveling straight to the target when possible and otherwise following the blocking obstacle's boundary. The algorithm Bug1 is based on exploring each encountered obstacle's complete boundary and leaving the obstacle at the point with the shortest distance to the target. The algorithm Bug2 is based on leaving obstacles when the line drawn from the start point to the target point is crossed. In \cite{Lumelsky1986} and \cite{Lumelsky1987},  upper bounds of the path lengths generated by the algorithms Bug1 and Bug2 were given in terms of the sum of perimeters of the unknown obstacles. The greediest algorithm, Com, is not guaranteed to reach the target and can enter an infinite loop. Other variations of Com, Bug1, and Bug2 have been developed; see, for example,  \cite{Sankaranarayanar1990, Sankaranarayanan1990b, Horiuchi2001, Kamon1998, Taylor2009, Lee1997}. The reader is referred to \cite{Mcguire2019} for an extensive survey of bug algorithms.

The authors in \cite{Sarid2007} and \cite{Kandathil2020} presented multi-robot extensions of the algorithm Bug1 and focused on their theoretical characteristics. In \cite{Sarid2007}, the robots are divided into pairs, and each pair is assigned an elliptic curve in which they travel toward the target. When a pair encounters an obstacle, the pair splits, and the robots explore one side of the obstacle each. The pair leaves the obstacle from the point closest to the target, similar to Bug1. A worst-case scene was constructed in \cite{Sarid2007}, in which the path lengths of the proposed algorithm were compared against that of an optimal off-line solution with full environmental knowledge. In  \cite{Kandathil2020}, for each encountered obstacle, two of the robots explore the obstacle boundary to identify the point closest to the target, and until this point is identified, the other robots stand still. This method has the advantage of minimizing the waiting robots' path lengths, which is important in energy-critical applications. The theoretical evaluation in \cite{Kandathil2020} is based on comparing the performances of the developed multi-robot algorithm and  Bug1 in environments with one obstacle.

Recently, in \cite{Mcguire2019_b, Duisterhof2021, Tan2022}, bug algorithms for swarm robotics were developed and experimentally tested. Under the algorithm in  \cite{Mcguire2019_b}, the robots in the swarm travel in different directions to cover different parts of the environment, which is helpful in, for example, search-and-rescue applications. The method in \cite{Duisterhof2021} was developed for seeking gas leaks using a swarm of robots, and the gas intensity is used to get the search directions of the robots. The algorithm in \cite{Tan2022} was developed for search missions and uses robots equipped with auditory and olfactory sensors to support cooperation within the swarm. Performance guarantees were not provided in  \cite{Mcguire2019_b, Duisterhof2021, Tan2022}.

\subsection{Contributions}

In this paper, we develop swarm extensions of the classical bug algorithms Com, Bug1, and Bug2 and give performance guarantees of the developed algorithms. To the best of our knowledge, the only existing literature providing swarm extensions of bug algorithms with performance guarantees are \cite{Sarid2007} and \cite{Kandathil2020}. In \cite{Sarid2007} and \cite{Kandathil2020}, extensions of the Bug1 algorithm were proposed, similar to this paper, where a new swarm extension of Bug1 will be provided (together with extensions of Com and Bug2). The performance guarantees provided for the algorithms in this paper are in terms of the time to reach the target and will depend on the sum of the perimeter lengths of the obstacles. This is similar to the seminal work \cite{Lumelsky1986} and \cite{Lumelsky1987}, where Com, Bug1, and Bug2 were first proposed, but is different from \cite{Sarid2007}, where the path lengths are compared against an optimal off-line solution with full knowledge of the environment and also different from \cite{Kandathil2020}, where the authors only show that their swarm extension of Bug1 outperforms the original Bug1 in several aspects. 

 The main contributions of this paper are as follows:

\begin{itemize}
	\item We derive a universal lower bound of any path generation algorithm which is suitable for comparison against any developed swarm path generation algorithm.

    \item We formulate path generation algorithms SwarmCom, SwarmBug1, and SwarmBug2, which are swarm extensions of the classical single-agent path generation algorithms Com, Bug1, and Bug2. 
	
	\item We derive an upper bound of the path generation algorithm SwarmBug1. This upper bound turns out to coincide with the derived universal lower bound in the limit when the swarm size approaches infinity. No other algorithm thus has a strictly better worst-case performance than SwarmBug1, and the universal lower bound is tight.

	\item We derive a lower bound for the path generation algorithm SwarmBug2. The lower bound for SwarmBug2 is higher than the upper bound for SwarmBug1, indicating that SwarmBug1 is preferable. 
	
\end{itemize}

\section{SYSTEM MODEL}\label{sec:sm}

We consider a swarm of $n$ agents whose task is to, as quickly as possible, get one agent from the swarm's starting point to a fixed target point. The task is completed once one agent in the swarm has arrived at the target point. The start point and the target point, as well as the movement of the agents in the swarm, are located in a 2D environment. Each agent in the swarm knows its coordinates and the start point's and target point's coordinates; hence, it can also know its direction to the target as well as the line between the start and target points. The task is challenging as the environment also includes obstacles unknown to the agents  \emph{a priori}. The agents only have tactile sensing to make the problem even more challenging. That is, agents cannot sense obstacles at a distance. The obstacles are defined by closed curves and are non-overlapping. We also assume a path exists from the start point to the target point. An example of an environment with obstacles is illustrated in Fig.~\ref{fig:general}.

\begin{figure}
	\centering
	\includegraphics[width=0.4\textwidth]{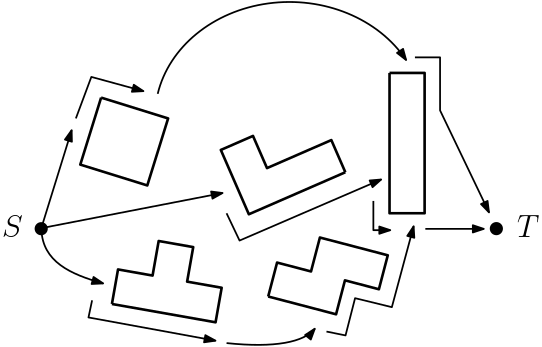}
	\caption{A swarm of three agents move from a start point ($S$) to a target point ($T$) while navigating a scene with obstacles in a 2D environment.}
	\label{fig:general}
\end{figure}

The agents in the swarm can move toward the target in a straight line (if an obstacle does not block this direction) or stand still. The agents can also follow the boundary of an obstacle, either in a clockwise or counterclockwise direction, when located at one of the obstacle's boundary points. We assume that the agents' speeds are between zero and some maximum speed. The agents are modeled as points, and the spacing and collision avoidance between agents are not addressed in this work.

We introduce some necessary notations before proceeding. The length between the start point and the target point is denoted as $D$. The perimeter of an obstacle indexed with $i$ is denoted as $p_i$. The time to travel a distance $x$ with maximum speed is denoted as $t(x)$. The total time to traverse the perimeters of all obstacles in the environment with maximum speed is denoted as $\sum t(p_i)$. The time for the first agent in the swarm to arrive at the target point is denoted as $t_f$, and the time for the last agent to arrive is denoted as $t_l$.

\section{UNIVERSAL LOWER BOUND}\label{sec:lb}

In this section, we show a universal lower bound on the travel time for the first agent in the swarm to reach the target in a constructed  scene. This lower bound
applies to any path generation algorithm and is a fundamental limitation of the achievable upper bound of the travel time for the first agent in the swarm to reach the target.

\begin{theorem}\label{thm:universal}
	For any algorithm of path generation for a swarm of $n\geq 2 $ agents and any strictly positive $\delta$, there is a scene such that the travel time for the first agent to reach the target satisfies
	\begin{equation}\label{eq:proof1}
		t_f\geq  t(D)+ \tfrac{1}{2} \sum t(p_i)-\delta,
	\end{equation}
	where $t_f$, $t(D)$, and $\sum t(p_i)$ are already defined.
\end{theorem}

\begin{proof} The proof is based on showing that for any path generation algorithm for a swarm of $n\geq 2$ agents, the generated paths will satisfy \eqref{eq:proof1} for the constructed scene illustrated in Fig. \ref{fig:Proof1}. The constructed scene includes one rectangular obstacle with width $W$ and length $2L$. Since the scene only includes one obstacle, $\sum p_i$ includes only one term. The start point is located on the midpoint at one of the obstacle's boundary sides. The target point has the same $x$-coordinate as the start point but is located on the other side of the obstacle (not necessarily on a boundary point). We neglect the width of the obstacle in the constructed scene by setting $W:=0$. We have $\sum p_i=4L$ and the time it takes to follow the shortest path between the start point and the target point is $t(P^*)~=~t(L)+\sqrt{t(D)^2+t(L)^2}$. Thus, independent of the choice of path generation algorithm, the time it takes for the first agent in the swarm to reach the target obeys $t_f\geq t(P^*)= \sum t(p_i)/2-t(L)+\sqrt{t(D)^2+t(L)^2}$. By setting $D$ and $L$ to satisfy $\delta\geq t(D)+t(L)-\sqrt{t(D)^2+t(L)^2}$ or equivalently $\sqrt{t(D)^2+t(L)^2} \geq t(D)+t(L)-\delta$, we  achieve $t_f\geq t(D)+\frac12\sum t(p_i) -\delta$. \end{proof}
	
\begin{figure}
		\centering
		\includegraphics[width=3.0in]{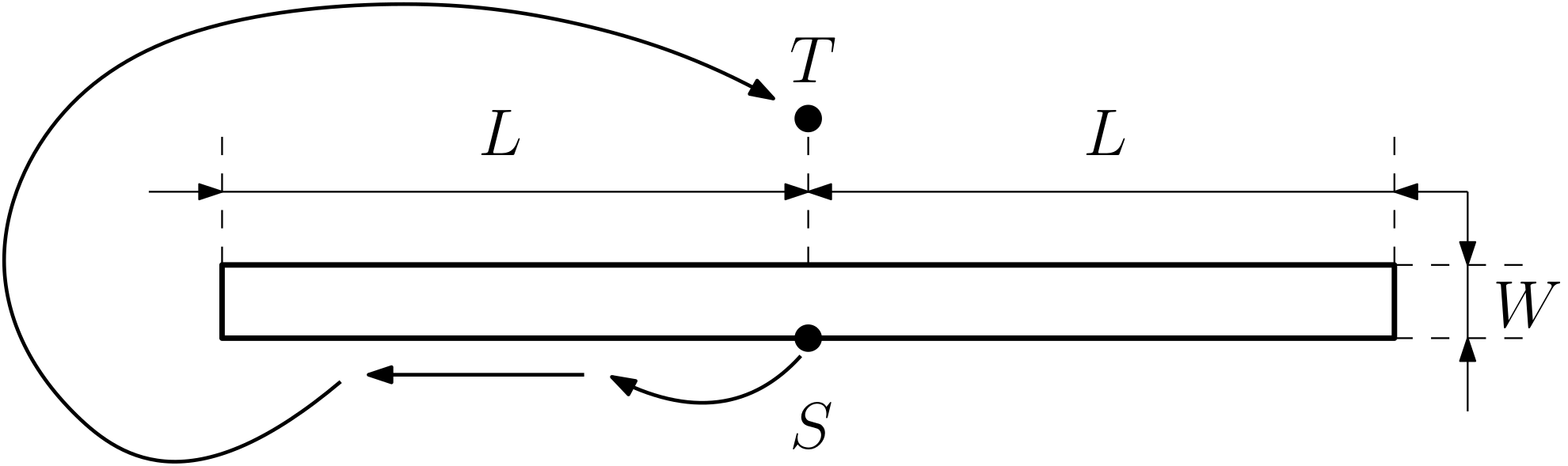}
		\caption{Constructed scene for Theorem 1. Notations in the figure are target point $(T)$, start point $(S)$, length of obstacle $(2L)$,  width of obstacle $(W)$, and the path of an agent in the swarm (arrows).}
		\label{fig:Proof1}
\end{figure}
	
\begin{remark} The universal lower bound in Theorem \ref{thm:universal} should be compared with the universal lower  bound for single agent path generation algorithms shown in \cite{Lumelsky1986} and \cite{Lumelsky1987}, which is $t_f~\geq~ t(D)~+~\sum t(p_i)-~\delta$.
\end{remark}

\begin{remark} The authors in \cite{Lumelsky1986} and \cite{Lumelsky1987} derived bounds of the path lengths of their algorithms. These bounds can be written regarding travel time if the single-agent travels at full speed along the generated paths. 
\end{remark}

\section{SwarmCom }\label{sec:sc}

The first path generation algorithm we propose, called SwarmCom, is inspired by Com in the seminal work \cite{Lumelsky1986} and \cite{Lumelsky1987} but extended to suit a swarm of agents instead of a single agent. The idea of SwarmCom (and of Com) is to travel towards the target in a straight line if an obstacle does not block this direction and otherwise follow the obstacle boundary clockwise or counterclockwise. Each time an obstacle is encountered, the swarm is divided into two groups; one group follows the boundary in a clockwise direction, and the other group follows the boundary in a counterclockwise direction. If only one agent (due to the swarm being divided at previous obstacles) encounters an obstacle, the agent takes either left or right, for example, by randomizing. The procedure of SwarmCom is as follows, see also Fig. \ref{fig:com}.

\begin{algorithm}[SwarmCom]

The swarm is initially located at the start point as one group.

\begin{enumerate}[Step 1:]
	\item Move in a straight line towards the target point until an obstacle is encountered. Then go to Step $2$. If the target is encountered, the procedure terminates. 
	
    \item  Split the group that encountered the obstacle into two groups. One group follows the boundary in clockwise direction and the other group in counterclockwise direction. Follow the boundary until the direction straight to the target is not blocked by the obstacle. Then go to Step 1.   

\end{enumerate}
\end{algorithm}

\begin{figure}
	\centering
	\includegraphics[width=1\linewidth]{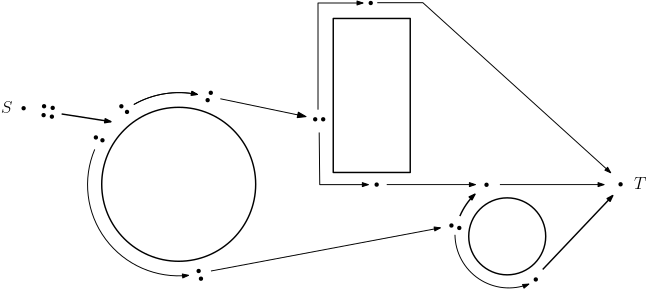}
	\caption{Illustration of the SwarmCom procedure.}
	\label{fig:com}
\end{figure}

Even though SwarmCom intuitively seems effective, there is no guarantee that any group of agents in the swarm will ever reach the target point. To illustrate this, consider the relatively simple scenario in Fig.~\ref{fig:Comfail}, including only one obstacle. The swarm will travel straight from the start point towards the target until it hits the obstacle. The swarm will then split into two groups. The group that traveled in the clockwise direction will end up at point $L_1$ where the direction straight to the target is not blocked and they will travel toward the target until they hit the obstacle again, and the group that traveled in the counterclockwise direction will have a similar behavior. The swarm will repeat this behavior indefinitely under SwarmCom and thus never reach the target.

\begin{figure}
	\centering
	\includegraphics[width=2in]{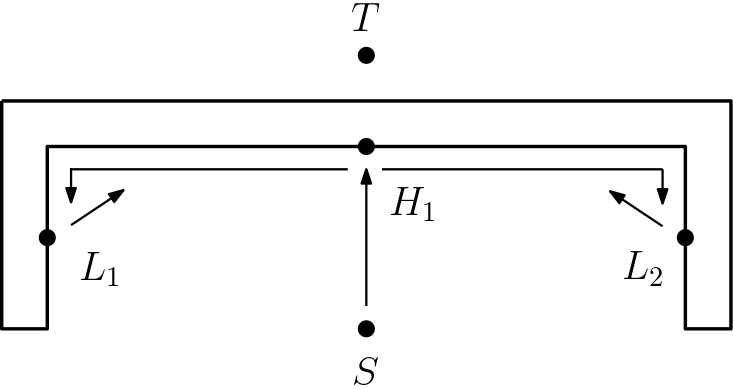}
	\caption{Trivial example scenario where SwarmCom fails to reach the target.}
	\label{fig:Comfail}
\end{figure}

\section{SwarmBug1 }\label{sec:sb1}

\subsection{Procedure of SwarmBug1}

The second path generation algorithm we propose, called SwarmBug1, is inspired by Bug1 in the seminal work \cite{Lumelsky1986} and \cite{Lumelsky1987} but extended to suit a swarm of agents instead of a single agent. The concept of SwarmBug1 (and of Bug1) is to extensively explore the complete boundary of each obstacle that the swarm encounters in the path from the start to the target point. The first point at obstacle $i$ that an agent in the swarm hits is denoted as $H_i$. The swarm leaves each obstacle~$i$ from its boundary point with a minimal distance to the target point. The point at which the swarm leaves obstacle~$i$ is denoted as $L_i$. 

The fact that the swarm includes many agents is exploited in two ways. The first is that the agents in the swarm are divided into pairs. One particular pair (the explorers) will explore an obstacle $i$ starting at $H_i$ by moving along its boundary in opposite directions (one will move in a clockwise direction and the other counterclockwise direction) and different to the explorer pair, the other pairs will not split up and instead move unified as illustrated in  Fig.~\ref{fig:ill3}.   The point at which the explorers meet after traversing half of the obstacle each is denoted as $M_i$. The second way the swarm is exploited is while the explorer agents explore the boundary of obstacle~$i$, the other pairs will spread out along the boundary of the obstacle to rapidly leave the obstacle once the explorers meet and $L_i$ is identified. More precisely, if the swarm of $n$ agents is located at obstacle $i$ and the explorers together have traversed a distance $x$ along the obstacle boundary, then the pairs aim to spread out along the obstacle boundary such that the distance to the closest neighboring pair (or an explorer) is $d(n,x)=2x/n$, as illustrated in Fig.~\ref{fig:ill3}. Note that the distance between neighboring pairs increases as the explorers move. In this way, once the explorers meet at $M_i$, the inter-pair distances are $2p_i/n$, and the path of the pair of agents closest to $L_i$ will be less than $p_i/n$. The pair of agents closest to $L_i$ will be assigned the explorers of the next obstacle. 

\begin{figure}
	\centering
	\includegraphics[width=0.3\textwidth]{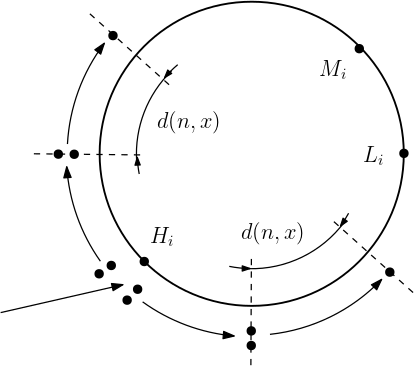}
	\caption{Five pairs of agents ($n=10$) arrive at obstacle $i$ and spread out over its boundary according to SwarmBug1 with the inter-pair distances $d(n,x)=2x/n$, where $x$ is the distance the explorers have traveled so far over the obstacle perimeter of obstacle $i$. The hit point $H_i$ is marked out in the figure as well as the meet point $M_i$ of the explorer pair and the leave point $L_i$, which will be identified once the explorers meet.}
	\label{fig:ill3}
\end{figure}

To execute the algorithm SwarmBug1, the explorers have to communicate and agree on which point along the boundary that is closest to the target point. Also, the swarm has to communicate and agree on which pair is closest to $L_i$ in order to assign a new pair as explorers. The procedure of SwarmBug1 is illustrated in Fig. \ref{fig:ill} and explained next.

\begin{algorithm}[SwarmBug1]\label{alg:SB1}
Let us first set $L_{0}=S$, the start point. The agents are divided into pairs and one pair of agents is set as explorers of the first obstacle.

\begin{enumerate}[Step 1:]
	\item The explorers of obstacle $i$ move toward the target in a straight line at full speed, starting at point $L_{i-1}$. The other agents aim to catch up with the explorers by taking the shortest already explored path at full speed from their positions to the explorers. When the explorers hit an obstacle, the hit point $H_{i}$ is defined. Then go to Step~$2$. If the explorers instead reach the target point, then the procedure terminates (the agents who have not yet arrived at $T$ will continue).	

	\item The explorers of obstacle $i$ move along the boundary of obstacle $i$ in opposite directions at full speed, starting at point~$H_{i}$. The pairs of agents that have reached point $H_{i}$ follow the explorers at full speed and each explorer is followed by every second pair.  The other agents aim to catch up by taking the shortest already  explored path at full speed from their positions to $H_{i}$. When all agents in the swarm have arrived at obstacle $i$, the agents aim to spread out to achieve equal distances between pairs, that is, distances $d(n,x)=2x/n$ between the pairs if the explorers have traveled a distance of $x$. When the explorers meet after having traversed half of the obstacle $i$'s boundary each, the meeting point $M_i$ and leave point $L_i$ are identified. Then go to Step 3. 
	
	\item  The pair of agents closest to $L_i$ is assigned the roles of explorers of obstacle $i+1$. All agents take the shortest already explored path from their positions to $L_{i}$. When the explorers of obstacle $i+1$ reach $L_{i}$,  go to Step 1.

\end{enumerate}
\end{algorithm}

\begin{figure}
	\centering
	\subfigure[Step 1 of SwarmBug1.]
	{
		\includegraphics[width=3.0in]{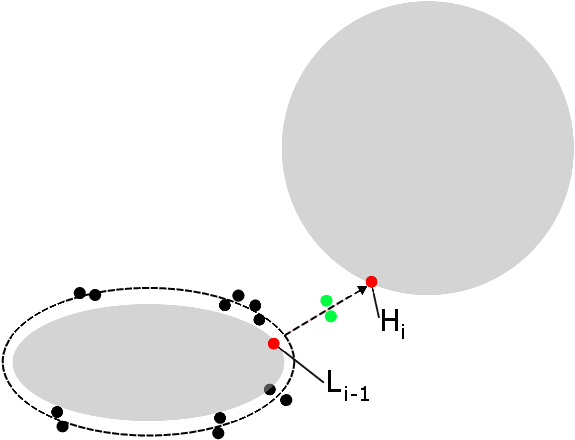}
		\label{fig:ill1}
	}
	\\
	\subfigure[Step 2 of SwarmBug1.]
	{
		\includegraphics[width=3.0in]{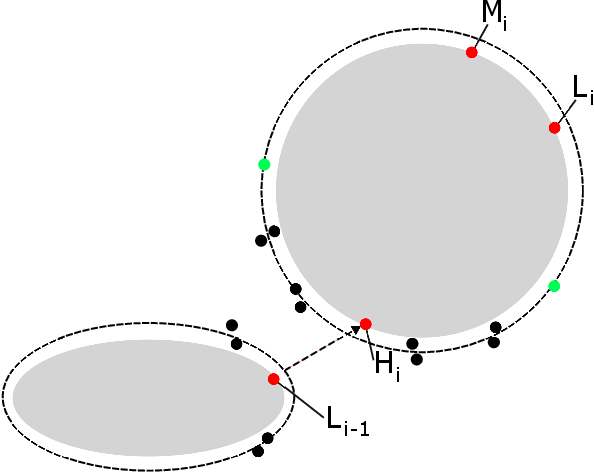}
		\label{ffig:ill2}
	}
	
		\subfigure[Step 3 of SwarmBug1.]
	{
		\includegraphics[width=3.0in]{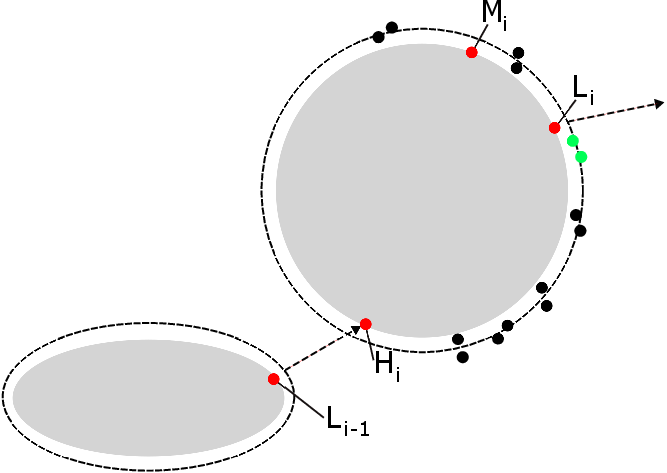}
		\label{ffig:ill3}
	}
	\caption{Illustration of the SwarmBug1 procedure. The steps are explained in Algorithm \ref{alg:SB1}.}
	\label{fig:ill}
\end{figure}

\subsection{Upper bounds of  SwarmBug1}

\begin{theorem}\label{thm:SB1}
Consider a swarm of $n$ agents using SwarmBug1. Then, the travel time for the first agent to reach the target satisfies

	\begin{equation}\label{eq:thm1a}
	t_f\leq t(D)+\tfrac12\sum \Big[t(p_i)+t\big(\min(2 \bar p/n,p_i )\big) \Big],
\end{equation}
where $ \bar p=\max\{p_i \}$ denotes the maximal obstacle perimeter in the environment. The travel time for the last agent to reach the target satisfies
\begin{equation}\label{eq:thm1b}
	t_l\leq t_f+  \tfrac12 t(\bar p).
\end{equation}

\end{theorem}
\begin{proof} SwarmBug1 will visit the same obstacles as Bug1 and have the same hit and leave points.	As Bug1 was shown in \cite{Lumelsky1986} and \cite{Lumelsky1987} to visit each obstacle at most once, the same holds for SwarmBug1. We will therefore omit this part in the proof. In this proof, we show an upper bound of the total time to explore the obstacles' boundaries and to reach the leave points under SwarmBug1.

Initially, the swarm is at the start point, and all agents in the swarm arrive at $H_{1}$ at the same time. The explorers traverse one side each of the obstacle $i$ at full speed, and equal distances between the pairs are achieved by the other pairs traveling with suitable fractions of the full speed. When the explorers meet at $M_1$, the distance between the pairs is  $2p_1/n$. The pairs will then travel to $L_1$ and in a straight line towards the target until they hit $H_2$. Once the explorers reach $H_2$, at least two pairs of agents will arrive at $H_2$ within each time period of $t(2p_1/n)$ time units until all agents have arrived at $H_2$. This follows from that the pairs approach $L_1$ from two directions as illustrated in Fig. \ref{fig:ill1} and the distances between the pairs approaching from each side are less than or equal to $2p_1/n$. Every second pair that arrives at $H_2$ after the explorers will go to the right, and every second pair will go left as illustrated in Fig. \ref{ffig:ill2}. All agents travel at full speed until all agents have arrived at the obstacle, and note that until then, the distances between neighboring pairs are less than or equal to $2p_1/n$ as all agents are traveling at the same speed. Once all agents have arrived at $H_2$, the pairs aim to achieve equal distances between the pairs. This can, for example, be achieved, at the latest, when the explorers traveled a distance of $2p_1/n \cdot n/2$, by that all agents, except for the explorers, first stand still, and each pair starts driving at full speed when the pair in front is driving and the distance to the in front driving pair is $2p_1/n$. Once the distances between the pairs are equal, the pairs drive with suitable fractions of the full speed to keep the equal distances. Thus, if all agents have not arrived at the second obstacle when the explorers meet, the distances between the pairs are $2p_1/n$. If all agents have arrived when the explorers meet, the distances between pairs are less than or equal to  $2p_2/n$. When the explorers arrive at $M_2$, the distance between pairs is thus less than or equal to $\max(2p_1/n, 2p_2/n)$.
 
The time for the explorers to travel from $H_2$ to $M_2$ is $t(p_2/2)$ as they travel half of the perimeter each. The time for the closest pair to arrive at the leave point $L_2$ once the explorers meet is less than or equal to $t(\max(2p_1/n, 2p_2/n))/2$, but also bounded by $t(p_2/2)$ as this is the maximum time for the explorers to reach the leave point. In this proof, we focused on obstacle $1$ and $2$. However, by taking identical steps, this procedure can be repeated to any obstacle $i$ and eventually until the agents reach the target point. Thus the inequality \eqref{eq:thm1a} is satisfied. 

Let $\bar{p}:= \max \{p_i\}_i$. When the explorers arrive at the target point, the distances between the pairs are less than or equal to $2\bar p/n$, and for each period of time length $t(2\bar p/n)$, it will arrive at least two pairs of agents to the target point until all agents have arrived. This follows from that the pairs approached the leave point from two directions at the obstacle when the distance between pairs $2\bar p/n$ was achieved. The number of periods needed for $n$ agents to arrive is less than or equal to $n/4$. Thus, $t_f-t_l \leq n/4 \cdot t(2\bar p/n)$ or $t_f \leq t_l +  t(\bar p/2)$. Thus the inequality \eqref{eq:thm1b} is satisfied.\end{proof}

\begin{remark} The upper bound of SwarmBug1 is strictly lower than the upper bound of Bug1 shown by the authors in \cite{Lumelsky1986} and \cite{Lumelsky1987}, which is $t_f~\leq~ t(D)~+~3/2 \cdot \sum t(p_i)$. This follows from the fact that the upper bound in Theorem \ref{eq:proof1} is less than or equal to $t(D)~+~\sum t(p_i)$, which occurs when the number of agents in the swarm is two ($n=2$).
	\end{remark}

\begin{corollary}\label{thm:coro1}
Consider a swarm of $n$ agents using SwarmBug1. When $n\rightarrow \infty$, the travel time for the first agent to reach the target is
\begin{equation}\label{eq:cor1}
t_f\leq t(D)+\tfrac12\sum t(p_i).
\end{equation}
\end{corollary}

\begin{proof}
Note that $2\bar p/n \rightarrow 0$ as $n\rightarrow \infty$. The conclusion follows from Theorem \ref{eq:proof1}. \end{proof}

\begin{remark}
	The upper bound in Corollary \ref{thm:coro1} coincides with the universal lower bound in Theorem \ref{eq:proof1} in the limit $\delta \rightarrow 0$. Hence, no path generation algorithm has a better worst-case performance than SwarmBug1 in the limit case when the swarm size approaches infinity ($n\rightarrow \infty$).   
\end{remark}

\begin{corollary}
Consider a swarm of $n$ agents using SwarmBug1. When all the obstacles have the same perimeter, the travel time for the first agent to reach the target satisfies
\begin{equation*}
	t_f\leq t(D)+ m t(p) \left(\frac{1}{2}+\frac{1}{n}\right),
\end{equation*}
where $m$ denotes the number of obstacles and $p$ denotes the perimeter of the obstacles.

	\end{corollary}

\begin{proof} Note that if all obstacles have the same perimeter $p$, then  $\bar p=p_i=p$. The conclusion follows from Theorem~\ref{thm:SB1}. \end{proof} 

\section{SwarmBug2}\label{sec:sb2}

\subsection{Procedure of SwarmBug2}

The last path generation algorithm we propose in this paper is SwarmBug2, a swarm extension of  Bug2 in  \cite{Lumelsky1986} and \cite{Lumelsky1987}. The concept of SwarmBug2 (and of Bug2), is to leave obstacles when crossing the line intersecting $S$ and $T$ and then travel along this line towards $T$ until an obstacle is encountered, see Fig. \ref{fig:bug2}. The line intersecting $S$ and $T$ is called the $M$-line and it is known by the agents as they know $S$ and $T$. Moreover, the agents know when they cross the $M$-line as they know their positions. The number of intersections between the $M$-line and the boundary of obstacle $i$ is denoted $m_i$. The number $m_i$ is even as closed curves define the obstacles. Due to the logic of SwarmBug2, each point that is an obstacle boundary  intersection with the $M$-line is either a leave point, a hit point, or neither. SwarmBug2 is a more greedy path generation algorithm than SwarmBug1, which explores the complete boundary of each encountered obstacle but is less greedy than SwarmCom, which leaves obstacles as soon as the direction towards the target point is free. It is worth noting that, different from SwarmBug1, under SwarmBug2, there might be more than one hit point at each obstacle. Thus, instead of referring to hit and leave points at obstacle $i$, $H_i$ and $L_i$ refer to the $i$th hit and leave the point in the generated path, respectively. 

\begin{figure}
	\centering
	\includegraphics[width=0.60\linewidth]{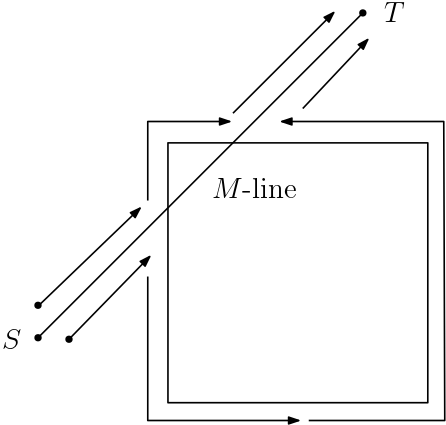}
	\caption{Illustration of the SwarmBug2 procedure.}
	\label{fig:bug2}
\end{figure}

Swarmbug2 exploits the fact that a swarm includes many agents through the splitting of a group (initially the whole swarm) into two when an obstacle is encountered.  The two groups leave the hit point and move around the obstacle in opposite directions (one group moves clockwise and one group moves counterclockwise).  The implementation of SwarmBug2 requires that the number of agents $n$ is large enough such that it allows us to split the swarm into two groups every time it arrives at a hit point (any group splits at most once at each hit point). More precisely, we require the number of agents to satisfy  
\begin{align*}
	n\geq2^{\frac12\sum m_i}.
\end{align*}  
The procedure of SwarmBug2 is given next.

\begin{algorithm}[SwarmBug2]\label{alg:swarmbug2} The swarm is initially located at the start point as one group.  Let us first set $L_0=S$ and $n_0=n$.	\begin{enumerate}[Step 1:] 
	
		\item The group at $L_{i-1}$ moves towards the target $T$ along the $M$-line until an obstacle is encountered and $H_{i}$ is defined. Then go to Step 2. If the group instead arrives at $T$,  the algorithm terminates.  
		\item At the hit point $H_{i}$, the group splits into two, one group of size $\lceil n_i/2\rceil$ and the other of size $\lfloor n_i/2\rfloor$, where $n_i$ denotes the number of agents in the group when it arrives at $H_{i}$. One of the two groups moves clockwise around the encountered obstacle and the other counterclockwise. Each group continues moving along the obstacle until a point on the $M$-line is encountered, say $X$, with the direction towards $T$ free and $\|X-T\|<\|H_{i}-T\|$. Then set  $L_{i}=X$ and go to Step~1.

	\end{enumerate}
\end{algorithm}

The algorithm SwarmBug2 is well-posed in the sense that it terminates with all agents having reached $T$ after a finite time. The time of the slowest group is bounded above as
\begin{align}\label{eq:cycles}
	t_l\leq t(D)+\tfrac12\sum m_it(p_i).
\end{align}
To show this, we can modify the algorithm Bug2 in \cite{Lumelsky1986} and \cite{Lumelsky1987} such that the direction (clockwise or counterclockwise) is chosen randomly at each hit point. Then, following the reasoning in \cite{Lumelsky1986} and \cite{Lumelsky1987}, it is easy to verify that the bound \eqref{eq:cycles} applies to this version of Bug2 for any realization of directions. The path of each group in the multi-agent case corresponds to one realization of the random algorithm.

\subsection{Bounds of SwarmBug2}

\begin{proposition}\label{thm:prop11}
Consider a swarm of $n$ agents using SwarmBug2. For any strictly positive $\varepsilon$, there is a scene consisting of a single obstacle such that the travel time for the first agent to reach the target satisfies
\begin{align}\label{eq:prop8}
t_f\geq t(D)+(1-\varepsilon)t(p).
\end{align}
\end{proposition}

\begin{proof}Consider the scene in Fig.~\ref{fig:zigzag}, consisting of one obstacle with perimeter $p$ and $k$ number of `combs', which are parts of the obstacle with arbitrarily large lengths. Here, $p$ is considered to be the inner circumference of the obstacle since there is no way for the agents to reach the outside. We construct the scene such that the horizontal path in each comb has length $p/(k+1)$ and the rest of the obstacle also has length $p/(k+1)$. The vertical part of the combs are accounted for by the parameter $D$. This scene design is possible as we can arbitrarily select the horizontal path length in each comb and select the number of comb teeth in each comb and the intermediate distances between the comb teeth and their widths. For the considered scene, at each hit point $H_j$, there is no benefit in going in the direction that avoids comb~$j$ (for example, there is no benefit in going left at $H_3$ in Fig.~\ref{fig:zigzag}) since this would incur an extra length of at least $p(j-1)/(k+1)$ for bypassing the $j$th comb whose length is $p/(k+1)$. Therefore, the path for the fastest group to reach the target under SwarmBug2 is $S$-$H_1$-$H_2$-$L_2$-$H_3$-$L_3$-\ldots-$H_k$-$L_k$-$T$, which only bypasses the first comb as this does not incur an extra length. The length of the total horizontal path in combs number $2$ to $k$ is $p(k-1)/(k+1)$. Hence, the travel time for the first group to reach the target is 
\begin{align*}
t_f&\geq t(D)+t(p)(k-1)/(k+1)\\
&=t(D)+(1-2/(k+1))t(p),
\end{align*}
as the vertical movement is at least $D$. The parameter $k$ is set to satisfy $\varepsilon\geq 2/(k+1)$. Thus, we  have  $t_f\geq t(D)+(1-\varepsilon)t(p)$ and the conclusion follows.\end{proof}

\begin{figure}
\centering
\includegraphics[width=0.231\textwidth]{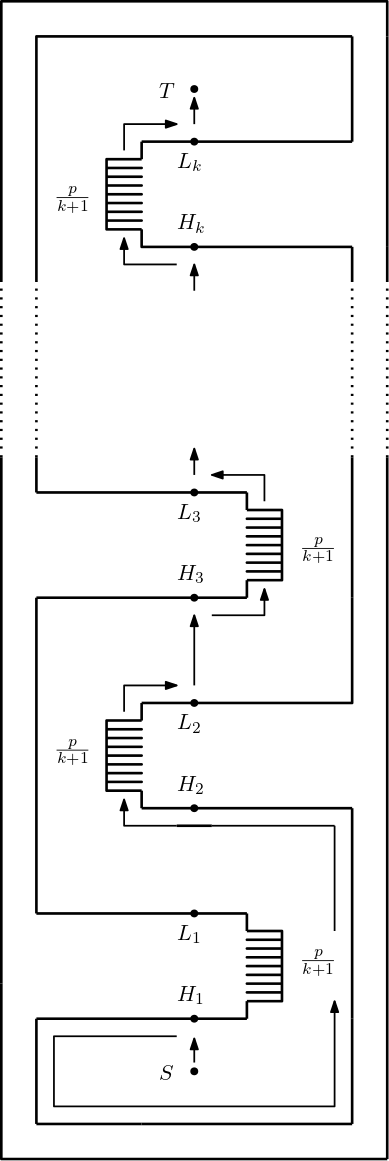}
\caption{Constructed scene for Proposition \ref{thm:prop11}. The trajectory ($S$-$H_1$-$H_2$-$L_2$-$H_3$-$L_3$-\ldots-$H_k$-$L_k$-$T$) of the group with the shortest path under SwarmBug2 is marked out.  }
\label{fig:zigzag}
\end{figure}

\begin{remark} Consider the scene in Fig. \ref{fig:zigzag}. The lower bound of the time for the first agent in the swarm to reach the target under SwarmBug2 in the limit $\varepsilon\rightarrow0$ is larger than the upper bound of SwarmBug1 given by \eqref{eq:thm1a}. Therefore, SwarmBug1 outperforms SwarmBug2 in terms of worst-case performance, and we thus do not derive any upper bound for the fastest group of SwarmBug2.
\end{remark}

Next, we give an upper bound of the time for the fastest group to reach the target when the $M$-line from $S$ to $T$ intersects each obstacle at a maximum of $2$ points ($m_i=2$). This holds, for example, when obstacles are convex.

\begin{proposition}\label{th:intersects}
Consider a swarm of $n$ agents using SwarmBug2.  Suppose that the $M$-line from $S$ to $T$ intersects each obstacle at a maximum of $2$ points. Then, the travel time of the first group to reach the target  satisfies
\begin{align}
t_f&\leq t(D)+\tfrac12\sum t(p_i).\label{eq:lower}
\end{align}
Moreover, the travel time of the slowest group satisfies
\begin{align*}
t_l&\leq t_f+\tfrac12\sum t(p_i).
\end{align*}
\end{proposition}

\begin{proof}
Under the assumptions, for each obstacle, there can only be one hit point and one leave point, wherefore, there cannot be any cycles. The $M$-line divides the obstacle into two separate parts. The length of the shortest of the two resulting paths along the circumference of the obstacle is at most $p_i/2$. The fastest group will travel along this path for every obstacle. The longest of the two resulting paths is less than $p_i$. The slowest group will travel along this path for every obstacle.
\end{proof}

\begin{remark} By comparing the upper bounds in Theorem~\ref{thm:SB1} and Proposition~\ref{th:intersects}, SwarmBug2 has strictly better worst-case performance than SwarmBug1 for a finite number of agents in the particular case when the $M$-line from $S$ to $T$ intersects each obstacle in a maximum of $2$ points. However, as has already been discussed, SwarmBug1 has a strictly better worst-case performance than SwarmBug2 in general cases.  
\end{remark}

Now we consider the more general case of scenes with only out-obstacles. An obstacle is called an out-obstacle if  $S$ and $T$ do not belong to its convex hull. An obstacle for which this does not hold is referred to as an in-obstacle. For example, an out-obstacle scene could be an outdoor urban environment where the obstacles are buildings, and an in-obstacle scene could be an indoor scene with walls enclosing the $S$ and $T$ points.

\begin{proposition}\label{th:out-obstacle}
Consider a swarm of $n$ agents using SwarmBug2.  Suppose that all obstacles are out-obstacles. Then, the travel time of the first group to reach the target  satisfies
\begin{align*}
t_f\leq t(D)+\tfrac12\sum t(p_i).
\end{align*}
\end{proposition}

\begin{proof}Under SwarmBug2, as the groups split at each hit point, most groups will sometimes take left and sometimes take right at hit points. However, one group will always stay on the left side of the $M$-line, and another will always stay on the right side. These two groups will have paths identical to the paths of Bug2, where the agent either always takes right or always left at hit points. The authors in  \cite{Lumelsky1986} and \cite{Lumelsky1987} showed that, for Bug2, in the case of an out-obstacle, the generated path will lie on one side of the $M$-line. Let the $i$th obstacle be split such that a fraction $\lambda_i\in [0,1]$ of the perimeter lies to the left and $1-\lambda_i$ lies to the right of the $M$-line. One group will stay on the shortest side of the $i$th obstacle every time it is encountered, and the fraction of the $i$th obstacle boundary this group will travel is  $\min \{\lambda_i,1-\lambda_i\}\leq1/2$.\end{proof} 

For the out-obstacle case, we can only apply the bound \eqref{eq:cycles} for the last group to reach the target. An example scene where the travel time for the last group is the upper bound \eqref{eq:cycles} is shown in Fig.~\ref{fig:long}.

\begin{figure}
	\centering
	\includegraphics[width=1\linewidth]{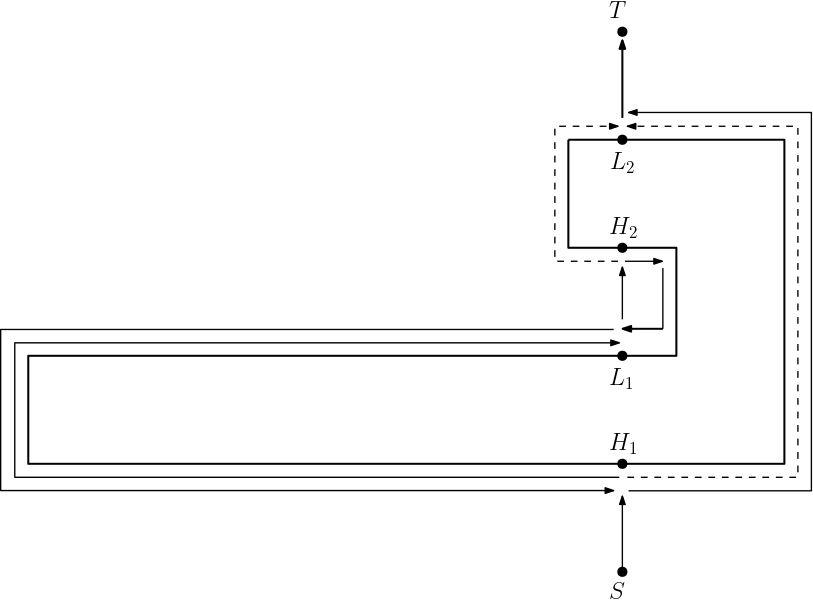}
	\caption{Example scene where the travel time for the last group under SwarmBug2 is the upper bound \eqref{eq:cycles}. 	The slowest group will travel the path $S$-$H_1$-$L_1$-$H_2$-$L_1$-$H_1$-$L_2$-$T$, which is twice along the path $H_1$-$L_1$. The obstacle can be constructed such that the part $H_1$-$L_1$ consists of an arbitrarily large proportion of the obstacle perimeter $p$. Hence, $t_l\geq t(D)+2t(p)-\delta=t(D)+m_i/2\cdot t(p)-\delta$ for any $\delta>0$ as in \eqref{eq:cycles}.}
	\label{fig:long}
\end{figure}

\section{CONCLUSIONS}\label{sec:cfw}


The path generation algorithms SwarmCom, SwarmBug1, and SwarmBug2 were presented, which are swarm extensions of the classical path generation algorithms Com, Bug1, and Bug2. As in the literature for the classical algorithms, we derived bounds for the swarm algorithms' performance in terms of the perimeters of the obstacles. We derived a universal lower bound of the travel time in a constructed scene which applies to any path generation algorithm and thus works as a fundamental limitation of what is possible to achieve in terms of worst-case performance. We also derived upper bounds of the worst-case performance for the algorithms SwarmBug1 and SwarmBug2. The worst case for SwarmCom is unbounded.

We found that swarm bug algorithms perform better in terms of arrival time at the target point $T$ than single bug algorithms. Which swarm bug algorithm is best is not given since the performance varies with the scene. The more greedy algorithms SwarmCom and SwarmBug2 have better performance for simple scenes, although their worst-case performance is significantly worse than that of SwarmBug1. For the greediest algorithm, SwarmCom, the swarm is not guaranteed to reach $T$. This indicates that greediness does not pay off for the problem of navigation in unknown environments. In future work, we plan to evaluate the path generation algorithms SwarmCom, SwarmBug1, and SwarmBug2 in simulations and real-world experiments.

\bibliographystyle{unsrt}

\bibliography{refs2} %

\begin{thebibliography}{10}

\bibitem{Dorigo2020}
M.~Dorigo, G.~Theraulaz, and V.~Trianni.
\newblock Reflections on the future of swarm robotics.
\newblock {\em Science Robotics}, 5(49):eabe4385, 2020.

\bibitem{Durrant2006}
H.~Durrant-Whyte and T.~Bailey.
\newblock Simultaneous localization and mapping: part {I}.
\newblock {\em IEEE Robotics \& Automation Magazine}, 13(2):99--110, 2006.

\bibitem{Vikgren2023}
M.~Vikgren and J.~Markdahl.
\newblock {tinySLAM-based exploration with a swarm of nano-UAVs}.
\newblock In {\em 6th International Symposium on Swarm Behavior and
  Bio-Inspired Robotics}, pages 899--904, 2023.

\bibitem{Lumelsky1986}
V.J. Lumelsky and A.A. Stepanov.
\newblock Dynamic path planning for a mobile automaton with limited information
  on the environment.
\newblock {\em IEEE Transactions on Automatic Control}, 31(11):1058--1063,
  1986.

\bibitem{Lumelsky1987}
V.J. Lumelsky and A.A. Stepanov.
\newblock Path-planning strategies for a point mobile automaton moving amidst
  unknown obstacles of arbitrary shape.
\newblock {\em Algorithmica}, 2(1-4):403--430, 1987.

\bibitem{Sankaranarayanar1990}
A.~Sankaranarayanar and M.~Vidyasagar.
\newblock Path planning for moving a point object amidst unknown obstacles in a
  plane: a new algorithm and a general theory for algorithm development.
\newblock In {\em Proceedings of the 29th IEEE Conference on Decision and
  Control}, pages 1111--1119, 1990.

\bibitem{Sankaranarayanan1990b}
A.~Sankaranarayanan and M.~Vidyasagar.
\newblock A new path planning algorithm for moving a point object amidst
  unknown obstacles in a plane.
\newblock In {\em Proceedings of the 7th IEEE International Conference on
  Robotics and Automation}, pages 1930--1936, 1990.

\bibitem{Horiuchi2001}
Y.~Horiuchi and H.~Noborio.
\newblock Evaluation of path length made in sensor-based path-planning with the
  alternative following.
\newblock In {\em Proceedings of the 17th IEEE International Conference on
  Robotics and Automation}, volume~2, pages 1728--1735, 2001.

\bibitem{Kamon1998}
I.~Kamon, E.~Rimon, and E.~Rivlin.
\newblock {TangentBug: A range-sensor-based navigation algorithm}.
\newblock {\em The International Journal of Robotics Research}, 17(9):934--953,
  1998.

\bibitem{Taylor2009}
K.~Taylor and S.M. LaValle.
\newblock {I-Bug: An intensity-based bug algorithm}.
\newblock In {\em Proceedings of the 26th IEEE International Conference on
  Robotics and Automation}, pages 3981--3986, 2009.

\bibitem{Lee1997}
S.~Lee, T.M. Adams, and B.-Y. Ryoo.
\newblock A fuzzy navigation system for mobile construction robots.
\newblock {\em Automation in Construction}, 6(2):97--107, 1997.

\bibitem{Mcguire2019}
K.N. McGuire, G.C.H.E. {de Croon}, and K.~Tuyls.
\newblock A comparative study of bug algorithms for robot navigation.
\newblock {\em Robotics and Autonomous Systems}, 121:103261, 2019.

\bibitem{Sarid2007}
S.~Sarid, A.~Shapiro, and Y.~Gabriely.
\newblock {MRBUG: A competitive multi-robot path finding algorithm}.
\newblock In {\em Proceedings 2007 IEEE International Conference on Robotics
  and Automation}, pages 877--882. IEEE, 2007.

\bibitem{Kandathil2020}
J.J. Kandathil, R.~Mathew, and S.S. Hiremath.
\newblock {Development and analysis of a novel obstacle avoidance strategy for
  a multi-robot system inspired by the Bug-1 algorithm}.
\newblock {\em Simulation}, 96(10):807--824, 2020.

\bibitem{Mcguire2019_b}
K.N. McGuire, C.~De~Wagter, K.~Tuyls, H.J. Kappen, and G.C.H.E. de~Croon.
\newblock Minimal navigation solution for a swarm of tiny flying robots to
  explore an unknown environment.
\newblock {\em Science Robotics}, 4(35):eaaw9710, 2019.

\bibitem{Duisterhof2021}
B.P. Duisterhof, S.~Li, J.~Burgu{\'e}s, V.J. Reddi, and G.C.H.E. de~Croon.
\newblock Sniffy bug: A fully autonomous swarm of gas-seeking nano quadcopters
  in cluttered environments.
\newblock In {\em Procedings of the 34th IEEE/RSJ International Conference on
  Intelligent Robots and Systems}, pages 9099--9106. IEEE, 2021.

\bibitem{Tan2022}
S.~Tan, X.~Zhang, J.~Li, R.~Jing, M.~Zhao, Y.~Liu, and Q.~Quan.
\newblock Oa-bug: An olfactory-auditory augmented bug algorithm for swarm
  robots in a denied environment.
\newblock {\em arXiv preprint arXiv:2209.14007}, 2022.

\end{thebibliography}

\end{document}